\documentclass[12pt]{article}

\usepackage[colorlinks]{hyperref}            
\usepackage{color}
\usepackage{graphicx,subfigure,amsmath,amssymb,amsfonts,bm,epsfig,epsf,url,dsfont}
\usepackage{times}
\usepackage{bbm}      
\usepackage{booktabs}
\usepackage{cases}
\usepackage{fullpage}
\usepackage[small,bf]{caption}
\usepackage{natbib}
\usepackage[top=1in,bottom=1in,left=1in,right=1in]{geometry}

\newcommand{\conv}{Conv}

\renewcommand{\phi}{\varphi}
\newcommand{\tz}{\tilde{z}}

\newcommand{\E}{\mathbb{E}}

\newcommand{\R}{\mathbb{R}}

\newcommand{\cK}{\mathcal{K}}

\newcommand{\cZ}{\mathcal{Z}}
\newcommand{\cA}{\mathcal{A}}

\newcommand{\cE}{\mathcal{E}}

\def\ds1{\mathds{1}}
\renewcommand{\epsilon}{\varepsilon}

\newcommand{\argmin}{\mathop{\mathrm{argmin}}}

\renewcommand{\tilde}{\widetilde}

\newlength{\minipagewidth}
\newcommand{\bookbox}[1]{
\par\medskip\noindent
\framebox[\textwidth]{
\begin{minipage}{\minipagewidth}
{#1}
\end{minipage} } \par\medskip }

\newcommand{\beq}{\begin{equation}}
\newcommand{\eeq}{\end{equation}}

\newcommand{\beqa}{\begin{eqnarray}}
\newcommand{\eeqa}{\end{eqnarray}}

\newcommand{\beqan}{\begin{eqnarray*}}
\newcommand{\eeqan}{\end{eqnarray*}}

\def\ba#1\ea{\begin{align*}#1\end{align*}} 
\def\banum#1\eanum{\begin{align}#1\end{align}} 

\newtheorem{theorem}{Theorem}
\newcommand{\BlackBox}{\rule{1.5ex}{1.5ex}}  
\newenvironment{proof}{\par\noindent{\bf Proof\ }}{\hfill\BlackBox\\[2mm]}

\begin{document}

\title{Towards Minimax Policies for Online Linear Optimization with Bandit Feedback
}
\author{
S{\'e}bastien Bubeck \\
Department of Operations Research and Financial Engineering, \\
Princeton University \\
{\tt sbubeck@princeton.edu} \\ \\
Nicol{\`o} Cesa-Bianchi \\
Dipartimento di Scienze dell'Informazione, \\
Universit{\`a} degli Studi di Milano\\
{\tt nicolo.cesa-bianchi@unimi.it}  \\ \\
Sham M. Kakade \\
Microsoft Research New England, and Department of Statistics, \\
Wharton School, University of Pennsylvania \\
{\tt skakade@microsoft.com} 
}

\date{\today}

\maketitle

\begin{abstract}
We address the online linear optimization problem with bandit feedback. Our contribution is twofold. First, we provide an algorithm (based on exponential weights) with a regret of order $\sqrt{d n \log N}$ for any finite action set with $N$ actions, under the assumption that the instantaneous loss is bounded by $1$. This shaves off an extraneous $\sqrt{d}$ factor compared to previous works, and gives a regret bound of order $d \sqrt{n \log n}$ for any compact set of actions. Without further assumptions on the action set, this last bound is minimax optimal up to a logarithmic factor. Interestingly, our result also shows that the minimax regret for bandit linear optimization with expert advice in $d$ dimension is the same as for the basic $d$-armed bandit with expert advice. Our second contribution is to show how to use the Mirror Descent algorithm to obtain computationally efficient strategies with minimax optimal regret bounds in specific examples. More precisely we study two canonical action sets: the hypercube and the Euclidean ball. In the former case, we obtain the first computationally efficient algorithm with a $d \sqrt{n}$ regret, thus improving by a factor $\sqrt{d \log n}$ over the best known result for a computationally efficient algorithm. In the latter case, our approach gives the first algorithm with a $\sqrt{d n \log n}$ regret, again shaving off an extraneous $\sqrt{d}$ compared to previous works.
\end{abstract}

\section{Introduction}
In this paper we consider the framework of online linear optimization: at each time instance $t = 1, \hdots, n$, the player chooses, possibly in a randomized way, an action from a given compact action set $\cA \subset \R^d$. The action chosen by the player at time $t$ is denoted by $a_t \in \cA$. Simultaneously to the player, the adversary chooses a loss vector $z_t \in \cZ \subset \R^d$ and the loss incurred by the forecaster is $a_t^{\top} z_t$. The goal of the player is to minimize the expected cumulative loss $\E \sum_{t=1}^n a_t^{\top} z_t$ where the expectation is taken with respect to the player's internal randomization (and possibly the adversary's randomization). In the basic version of this problem, the player observes the adversary's move $z_t$ at the end of round $t$. We consider here the {\em bandit} version, where the player only observes the incurred loss $a_t^{\top} z_t$. As a measure of performance we define the regret of the player as
$$R_n = \E \sum_{t=1}^n a_t^{\top} z_t - \min_{a \in \cA} \E \sum_{t=1}^n a^{\top} z_t~.$$ 
In this paper we are interested in the dual setting, where the adversary plays on a dual action set, i.e., $\cA$ and $\cZ$ are such that $|a^{\top} z| \leq 1, \forall (a,z) \in \cA \times \cZ$.

\subsection{Contributions and relation to previous works}
In the full information case, the online optimization setting (for convex losses) was introduced by \cite{Zin03}. The specific online linear optimization problem with bandit feedback was first studied by \cite{MB04} and \cite{AK04}. Our first contribution to this problem is to complete the research program started by \cite{DHK08} and~\cite{CL11}. In these papers the authors studied the {\sc exp2} (Expanded Exp) algorithm, also called Geometric Hedge, Expanded Hedge, or ComBand. This strategy applies to a finite set of actions; it assigns an exponential weight to each action, and then draws an action at random from the corresponding probability distribution. Using a basic estimation procedure (first used by \cite{ACFS03} for the basic multi-armed bandit problem), one can estimate the loss vector $z_t$. However, to control the range of the estimates, one has to mix the probability given by {\sc exp2} with an "exploration distribution". \cite{DHK08} chose this distribution to be uniform over a barycentric spanner for the action set, while in \citep{CL11} the distribution was uniform over all actions. Using ideas from convex geometry, we propose a new distribution that allows us to derive a minimax optimal regret bound. More precisely, we show that for any finite action set, {\sc exp2} with the exploration distribution given by John's Theorem (see Theorem \ref{th:john}) attains a regret of order $\sqrt{d n \log N}$ for any set of $N$ actions. This improves by a factor $\sqrt{d}$ over previous works. Moreover this rate is optimal: there exists action sets (such as the hypercube) where the minimax rate is of order $d \sqrt{n}$ ---see \citep{DHK08}. Surprisingly, this result also shows that {\sc Exp2} with John's exploration can be used for linear bandits with $N$ experts to obtain a regret of order $\sqrt{d n \log N}$, which is no worse than the minimax regret for the basic $d$-armed bandit with $N$ experts problem.

While these results show that, without further assumption on the set of action, the regret of {\sc exp2} is optimal, they do not say anything about optimality for a {\em specific} set of actions. In fact, it was proven by \cite{ABL11} that for some pair $(\cA, \cZ)$ the exponential weights is a provably suboptimal strategy (with a gap of order $\sqrt{d}$). 
To address this issue, another class of algorithms has been studied for online optimization:  the Mirror Descent style algorithms of \cite{NY83} ---this class of algorithms was rediscovered in the learning community, see for example \cite{KW01}. In recent years the number of papers using Mirror Descent to solve problems in online optimization has been growing very rapidly. In the full information setting (when one observes $z_t$), we have a very good understanding of how to use Mirror Descent to obtain optimal regret bounds that adapt to the geometry of the problem ---see \citep{Rak09, Haz11, Bub11}. In particular, a recent paper suggests that in this basic setting Mirror Descent is "universal", see \citep{SST11}. On the other hand, in the limited feedback scenario the picture is much more scattered. In the particular cases of {\em semi-bandit} feedback ---see \citep{ABL11}--- and {\em two-points bandit} feedback ---see \citep{ADX10}, we know how to use Mirror Descent to obtain optimal regret bounds. However, in both scenarios the feedback is much stronger than in the more fundamental bandit problem. In this latter case, there is only one paper that successfully applies Mirror Descent, namely the seminal work of \cite{AHR08} ---see also the follow-up paper \cite{AR09}. Unfortunately, for a convex and compact set $\cA$, this approach (which combines Mirror Descent with a self-concordant barrier for the action set) leads to a regret bound of order $d \sqrt{ \theta n \log n}$ for any $\theta > 0$ such that $\cA$ admits a $\theta$-self concordant barrier. For example, in the case of the hypercube the best we know is $\theta = O(d)$, which results in the suboptimal $d^{3/2} \sqrt{n \log n}$ regret (compared to $d \sqrt{n}$ for {\sc exp2} with John's ellipsoid). However, note that in this particular case it is not known if {\sc exp2} can be implemented efficiently, while Mirror Descent is polynomial time. 

Our second main contribution is to propose an efficient algorithm based on Mirror Descent, with an optimal regret bound for two canonical pairs $(\cA, \cZ)$. Namely, the (hypercube, cross-polytope) pair, which corresponds to an $L_{\infty}/L_1$ type of constraints, and the (Euclidean ball, Euclidean ball) pair, which corresponds to an $L_2/L_2$ constraint. In the former case this results in the first computationally efficient algorithm with a regret of order $d \sqrt{n}$, while in the latter case it is the first efficient algorithm with a regret of order $\sqrt{d n \log n}$. Indeed, the approach of \cite{AHR08} only gives $d \sqrt{n \log n}$ for the pair (Euclidean ball, Euclidean ball) since there exists a $O(1)$-self concordant barrier for the Euclidean ball. Note also that this specific example was studied in \cite{AR09}, we discuss their result in Section \ref{sec:ball}.


\subsection{Outline of the paper}
The paper is organized as follows. In Section~\ref{sec:alg} we introduce the two algorithms discussed in the paper: Expanded Exp ({\sc exp2}) and Online Stochastic Mirror Descent ({\sc osmd}). In both cases we state a general regret bound.  In Section~\ref{sec:john} we detail our exploration strategy for {\sc exp2}, and show the corresponding regret bound. We also discuss briefly the extension to linear bandits with expert advice. Then in Section~\ref{sec:hypercube} (respectively Section \ref{sec:ball}) we show how to use {\sc osmd} to obtain a computationally efficient strategy with optimal regret for the hypercube (respectively for the Euclidean ball, up to a logarithmic factor).

\section{Algorithms} \label{sec:alg}
We briefly describe here the two algorithmic templates that we shall use in this paper. First, {\sc exp2} is described in Figure~\ref{fig:exp2}. The general regret bound for this algorithm is the following. The proof of this result follows a standard argument, see for example [Chapter 7, \cite{Bub11}].

\begin{figure}[t]
\bookbox{
\textbf{Algorithm:} {\em {\sc exp2} with exploration $\mu$.}
\\
\textbf{Parameters:} learning rate $\eta$; mixing coefficient $\gamma$; distribution $\mu$ over the action set $\cA$.

\medskip\noindent
Let $q_1=\big(\frac1{|\cA|},\hdots,\frac1{|\cA|}\big) \in \R^{|\cA|}$. For each round $t=1,2,\ldots,n$;
\begin{itemize}
\item[(a)]
Let $p_t = (1-\gamma) q_t + \gamma \mu$, and play $a_t \sim p_t$.
\item[(b)]
Estimate the loss vector $z_t$ by $\tilde{z_t} = P_t^+ a_t a_t^{\top} z_t,$ with $P_t = \E_{a \sim p_t}\bigl[a a^{\top}\bigr]$.
\item[(c)]
Update the exponential weights, for all $a \in \cA$, 
$$q_{t+1}(a) =\frac{\exp(- \eta a^{\top} \tz_t) q_{t}(a)}{\sum_{b \in\cA} \exp(- \eta b^{\top} \tz_t) q_{t}(b)}.$$ 
\end{itemize}
}
\caption{{\sc exp2} strategy for bandit feedback.}\label{fig:exp2}
\end{figure}

\begin{figure}[t]
\bookbox{
\textbf{Algorithm:} {\em {\sc osmd.}}
\\
\textbf{Parameters:} learning rate $\eta > 0$; regularization function $F : \R^d \rightarrow \R \cup \{+\infty\}$ with effective domain $\cA$, and such that the Legendre-Fenchel dual $F^*$ is differentiable on $\R^d$; perturbation scheme for step (a) below.

\medskip\noindent
Let $a_1 \in \argmin_{a \in \cA} F(a)$. For each round $t=1,2,\ldots,n$;
\begin{itemize}
\item[(a)] 
Play $\tilde{a}_t$ at random from some probability distribution $p_t$ over $\cA$ \\ ($\tilde{a}_t$ is a randomly perturbated version of $a_t$, see Section \ref{sec:hypercube} and Section \ref{sec:ball} for examples).
\item[(b)]
Estimate the loss vector $z_t$ by $\tilde{z}_t = P_t^+ \tilde{a}_t \tilde{a}_t^{\top} z_t,$ 
with $P_t = \E_{a \sim p_t} \bigl[a a^{\top}\bigr]$.
\item[(c)]
Let $a_{t+1}  = \nabla F^* \left( - \eta \sum_{s=1}^{t-1} \tilde{z}_s \right)$.
\end{itemize}
}
\caption{Online Stochastic Mirror Descent ({\sc OSMD}) for bandit feedback.}\label{fig:OSMD}
\end{figure}

\begin{theorem} \label{th:exp2}
Let $\cA$ be a finite set of $N$ actions. For the {\sc exp2} strategy, provided that
$\eta |a^{\top} \tilde{z}_t| \leq 1, \forall a \in \cA ,$
one has
$$R_n \leq 2 \gamma n + \frac{\log N}{\eta} + \eta\,\E \sum_{t=1}^n \sum_{a \in \cA} {p}_t(a) \bigl(a^{\top} \tilde{z}_t\bigr)^2~.
$$
\end{theorem}
%
%
Figure \ref{fig:OSMD} describes {\sc osmd} in the bandit setting. Note that step (c) can be written in several equivalent ways, such as a Follow The Regularized Leader equation, or a mirror gradient descent step if $F$ is a Legendre function. When written as a gradient descent step, one usually has to project back on $\cA$ (using the Bregman divergence associated to $F$).
Here the projection is implicit in the evaluation of $\nabla F^*$.
The following theorem states a general regret bound for {\sc osmd}. Recall that the Bregman divergence with respect to $F$ is defined as $D_F(x,y) = F(x) - F(y) - (x-y)^{\top}\nabla F(y)$, and the Legendre-Fenchel dual of $F$ is defined as $F^*(v) = \sup_{x \in \cA} x^{\top} v - F(x)$.
In the following, we write $x_1^t$ to denote $x_1 + \cdots + x_t$.
\begin{theorem} \label{th:osmd}
Let $\cA$ be a compact set of actions, and $F$ a function with effective domain $\cA$, and such that $F^*$ is differentiable on $\R^d$. Then {\sc osmd} satisfies (for any norm $\|\cdot \|$)
$$R_n  \leq  \frac{\sup_{a \in \cA} F(a) - F(a_1)}{\eta} + \frac{1}{\eta} \sum_{t=1}^n \E D_{F^*}\bigl(- \eta\tilde{z}_1^t, - \eta\tilde{z}_1^{t-1}\bigr)
 + \sum_{t=1}^n \E \big\|a_t - \E[\tilde{a}_t \mid a_t] \big\| \cdot \|z_t\|_* ~.$$
\end{theorem}
\begin{proof}
The proof is adapted from \cite{KST10}. Using Young's inequality, one obtains $\forall a \in \cA$
\begin{align*}
    - \eta \sum_{t=1}^n a^{\top} \tilde{z}_t
& \leq
    F(a) + F^*\left( - \eta \tilde{z}_1^n \right)
\\&=
    F(a) + F^*(0) + \sum_{t=1}^n \Bigl( F^*\left( - \eta \tilde{z}_1^t \right)  - F^*\left( - \eta \tilde{z}_1^{t-1} \right) \Bigr)
\\ &=
    F(a) + F^*(0) + \sum_{t=1}^n \left( \nabla F^*\left( - \eta \tilde{z}_1^{t-1} \right)^{\top} (- \eta \tilde{z}_t) + D_{F^*}\bigl(- \eta\tilde{z}_1^t, - \eta \tilde{z}_1^{t-1}\bigr) \right)
\\ &=
    F(a) + F^*(0) + \sum_{t=1}^n \left( - \eta a_t^{\top} \tilde{z}_t + D_{F^*}\bigl(- \eta \tilde{z}_1^t, - \eta \tilde{z}_1^{t-1} \bigr) \right)
\end{align*}
since $F^*(0)= - F(a_1)$. This shows that: 
$$\sum_{t=1}^n (a_t -a)^{\top} \tilde{z}_t \leq \frac{F(a) - F(a_1)}{\eta} + \frac{1}{\eta} \sum_{t=1}^n D_{F^*}\bigl(- \eta \tilde{z}_1^t, - \eta \tilde{z}_1^{t-1}\bigr) .$$
Taking into account the randomness induced by $\tilde{a}_t$ and $\tilde{z}_t$ is then an easy exercise, see for example \citep[Chapter 7]{Bub11}.
\end{proof}
This theorem proves to be particularly useful when applied with a Legendre function $F$ ---see \cite[Chapter 11]{CL06} for the definition of a Legendre function. Indeed, in that case $F^*$ is differentiable if $F$ is differentiable, and moreover the corresponding gradient mappings are inverse of each other, which gives a simple way to do computations with the Bregman divergence $D_{F^*}$.

%

\section{{\sc exp2} with John's exploration} \label{sec:john}
We propose here a new exploration distribution $\mu$ for the {\sc exp2} strategy, that allows us to derive the first $\sqrt{d n \log N}$ regret bound for online linear optimization with bandit feedback. We use the following result from convex geometry, see \citep{Bal97} for a proof.
\begin{theorem} \label{th:john}
Let $\cK \subset \R^d$ be a convex set. If the ellipsoid $\cE$ of minimal volume enclosing $\cK$ is the unit ball in some norm derived from a scalar product $\langle \cdot, \cdot \rangle$, then there exists $M \leq d (d+1)/2 + 1$ contact points $u_1, \hdots, u_M$ between $\cE$ and $\cK$, and $\mu \in \Delta_M$ (the simplex of dimension $M-1$), such that
$$x = d \sum_{i=1}^M \mu_i \langle x, u_i \rangle u_i, \forall x \in \R^d .$$
\end{theorem}
To use this theorem, we need to perform a preprocessing of the action set as follows:
\begin{enumerate}
\item First, we assume that $\cA$ is of full rank (that is such that linear combinations of $\cA$ span $\R^d$). If it is not the case, then one can rewrite the elements of $\cA$ in some lower dimensional vector space and work there.
\item Find John's ellipsoid for $\conv(\cA)$ ---i.e., the ellipsoid of minimal volume enclosing $\conv(\cA)$: $\cE = \{x \in \R^d: (x-x_0)^{\top} H^{-1} (x-x_0) \leq 1 \}$. The first preprocessing step is to translate everything by $x_0$. In other words, we assume now that $\cA$ is such that $x_0 = 0$. Furthermore, we define the inner product $\langle x,y\rangle = x^{\top} H y$. 
\item We can now assume that we are playing on $\cA' = H^{-1} \cA$, and the loss of playing $a' \in \cA'$ when the adversary plays $z$ is $\langle a', z\rangle$. Indeed:  $\langle H^{-1} a, z\rangle = a^{\top} z$. Moreover, note that John's ellipsoid for $\conv(\cA')$ is the unit ball for the inner product $\langle \cdot , \cdot \rangle$ because $\langle H^{-1} x, H^{-1} x\rangle = x^{\top} H^{-1} x$.
\item Find the contact points $u_1, \hdots, u_M$ and $\mu \in \Delta_M$ that satisfy Theorem \ref{th:john} for $\conv(\cA')$. Note that the contact points are in $\cA'$, thus they are valid points to play. We say that $\mu$ is John's exploration distribution.
\end{enumerate}
In the following we drop the prime on $\cA'$. More precisely. we play on a set $\cA$ such that John's ellipsoid for $\conv(\cA)$ is the unit ball for some inner product $\langle \cdot , \cdot \rangle$, and the loss is given by $\langle a, z\rangle$. Thus, we also need to slightly change the algorithm to account for the fact that the loss is now an arbitrary scalar product. Step (c) in Figure \ref{fig:exp2} is modified as:
$$q_{t+1}(a) =\frac{\exp(- \eta \langle a, \tz_t \rangle) q_{t}(a)}{\sum_{b \in\cA} \exp(- \eta \langle b, \tz_t \rangle) q_{t}(b)}.$$ 
We also modify the loss estimate given by step (b) as follows. Recall that the outer product $u \otimes u$  
is defined as the linear mapping from $\R^d$ to $\R^d$ such that $u \otimes u (x) = \langle u, x \rangle u$. Note that one can also view $u \otimes u$ as a $d \times d$ matrix, so that the evaluation of $u \otimes u$ is equivalent to a multiplication by the corresponding matrix. Now let:
$$P_t = \sum_{a \in \cA} {p}_t(a) a \otimes a .$$
Note that this matrix is invertible, since $\cA$ is of full rank and ${p}_t(a) > 0$, $\forall a \in \cA$. The estimate for $z_t$ is given by:
\begin{equation} \label{eq:banditestimate}
\tilde{z}_t = P_t^{-1} \left({a}_t \otimes {a}_t\right) z_t.
\end{equation}
Note that this is a valid estimate since $\left({a}_t \otimes {a}_t\right) z_t = \langle {a}_t, z_t \rangle {a}_t$ and $P_t^{-1}$ are observed quantities. Moreover, it is also clearly an unbiased estimate. We can now prove the following result.
\begin{theorem} \label{th:exp2john}
{\sc exp2} with John's exploration and estimate \eqref{eq:banditestimate} satisfies, for $\frac{\eta d}{\gamma} \leq 1$,
$$R_n \leq 2 \gamma n + \frac{\log N}{\eta} + \eta n d.$$
In particular with $\gamma = \eta d$ and $\eta = \sqrt{\frac{\log N}{3 n d}}$ we have that
$$R_n \leq 2 \sqrt{3 n d \log N} .$$
\end{theorem}
\begin{proof}
With the chosen scalar product, it is easy to see that the condition $\eta |a^{\top} \tilde{z}_t| \leq 1$ in Theorem \ref{th:exp2} rewrites as $\eta | \langle a, \tilde{z}_t \rangle | \leq 1$, while the third term in the regret bound rewrites as $\E \sum_{a \in \cA} {p}_t(a) \langle a, \tilde{z}_t \rangle^2$. Thus it remains to control those two quantities. Let us start with the latter:
\begin{align*}
    \sum_{a \in \cA} \tilde{p}_t(a) \langle a, \tilde{g}_t \rangle^2
&=
    \sum_{a \in \cA} {p}_t(a) \langle \tilde{z}_t, (a \otimes a) \tilde{z}_t \rangle 
\\ &=
    \langle \tilde{z}_t, P_t \tilde{z}_t \rangle
=
    \langle {a}_t, z_t \rangle^2 \langle P_t^{-1} {a}_t, P_t P_t^{-1} {a}_t\rangle
\leq
    \langle P_t^{-1} {a}_t, {a}_t \rangle .
\end{align*}
Now we use a spectral decomposition of $P_t$ in an orthonormal basis for $\langle \cdot, \cdot \rangle$ and write
$P_t = \sum_{i=1}^d \lambda_i v_i \otimes v_i .$
In particular, we have $P_t^{-1} = \sum_{i=1}^d \frac{1}{\lambda_i} v_i \otimes v_i$ and thus:
$$\E \langle P_t^{-1} {a}_t, {a}_t \rangle  =  \sum_{i=1}^d \frac{1}{\lambda_i} \E \langle (v_i \otimes v_i) {a}_t, {a}_t \rangle
 = \sum_{i=1}^d \frac{1}{\lambda_i} \E \langle ({a}_t \otimes {a}_t) v_i, v_i \rangle
 =  \sum_{i=1}^d \frac{1}{\lambda_i} \langle P_t v_i, v_i \rangle  =  d. $$
This concludes the bound for $\E \sum_{a \in \cA} {p}_t(a) \langle a, \tilde{z}_t \rangle^2$. We turn now to $\langle a, \tilde{z}_t \rangle$:
$$\langle a, \tilde{z}_t \rangle = \langle {a}_t, z_t \rangle \langle a, P_t^{-1} {a}_t \rangle
\leq \langle a, P_t^{-1} {a}_t \rangle
\leq \frac{1}{\min_{1 \leq i \leq d} \lambda_i}$$
where the last inequality follows from the fact that $\langle a, a \rangle \leq 1$ for any $a \in \cA$, since $\cA$ is included in the unit ball. Now to conclude the proof we need to lower bound the smallest eigenvalue of $P_t$. Using Theorem \ref{th:john}, one can see that $P_t \succeq \frac{\gamma}{d} I_d$, and thus $\lambda_i \geq \frac{\gamma}{d}$ concluding the proof.
\end{proof}
Using the discretization argument of \cite{DHK08}, {\sc exp2} with John's exploration can be used to obtain a regret of order $\sqrt{d n \log n}$ for any compact set of action $\cA$.

\subsection{Computational issues}
If $\cA$ is given by a finite set of points, then ~\cite{GroetschelLovaszSchrijver1993} give a polynomial time algorithm for computing a constant factor approximation to the John's ellipsoid (and this approximate basis will provide the same order of regret). However, if $\cA$ is specified by the intersection of half spaces, then  \cite{Nemi07} shows that obtaining such a constant factor approximation to this ellipsoid is NP-hard in general. Here, it is possible to efficiently compute an ellipsoid where the factor of $d$ in Theorem~\ref{th:john} is replaced by $d^{3/2}$ ---see~\citep{GroetschelLovaszSchrijver1993}, which leads to a slightly worse dependence on $d$ in the regret bound.

In special cases, we conjecture that the John's ellipsoid may be computed efficiently, as for certain problems, there are efficient implementations of GeometricHedge that lead to optimal rates (such as shortest path problems and other settings where dynamic programming solutions exists).

\subsection{Application to bandits with experts}
Consider the following model of linear bandits with $N$ experts. At each time step $t =1, 2, \hdots, n$, each expert $k = 1, \hdots, N$ suggests an action $a_t(k) \in \R^d$. The goal here is to compete with the best expert, that is at each time step the strategy chooses an expert $k_t \in \{1, \hdots, N\}$ and the regret is given by:
$$R_n^{\mathrm{exp}} = \E \sum_{t=1}^n a_t(k_t)^{\top} z_t - \min_{k \in \{1, \hdots, N\}} \E \sum_{t=1}^n a_t(k)^{\top} z_t . $$
One can use {\sc exp2} with John's exploration to obtain a regret of order $\sqrt{d n \log N}$ for this problem. Indeed, it suffices at every turn to do the preprocessing step on $\cA_t = \{a_t(1), \hdots, a_t(N)\}$ and to build the corresponding John's exploration $\mu_t$, the straightforward details are omitted.

For example, at each time $t$ each expert $i=1,\dots,N$ is associated with a hidden loss estimate $z_t(i)\in\cZ$ and an arbitrary ``context set'' $\cA_t\subseteq\cA$ is observed. Each expert $i$ then suggests the best action according to the current loss estimate,
$
	a_t(i) = \argmin_{a\in\cA_t} z_t(i)^{\top}a~.
$
This can be viewed as a natural nonstochastic variant of the contextual linear bandit model of~\cite{CLRS11}.
Another notable special case is the $d$-armed bandit problem with expert advice, where we can view the suggested actions as the corners of the $d$-dimensional simplex. Here, the EXP4 algorithm of~\cite{ACFS03} achieves a regret of order $\sqrt{d n \ln N}$. Interestingly, the regret achievable
in the more general $d$-dimensional linear optimization setting is no worse than in the seemingly simpler $d$-armed bandit with expert advice setting.

\section{Computationally efficient strategy for the hypercube} \label{sec:hypercube}
In this section we restrict our attention to the action set $\cA = \{x \in \R^d: \|x\|_{\infty} \leq 1 \}$. Using {\sc exp2} with John's exploration on $\{-1,1\}^d$ one obtains a regret bound of order $d \sqrt{n}$ for this problem, and as it was shown by \cite{DHK08} this regret is minimax optimal. However, it is not known if it is possible to sample from the exponential weights distribution in polynomial time for this particular set of actions. In this section we propose to turn to {\sc osmd}, and we show that with the appropriate regularizer $F$ and random perturbation $\tilde{a}_t$ (see step (a) in Figure \ref{fig:OSMD}), one can obtain a minimax optimal algorithm with computational complexity linear in $d$. 
More precisely we use an entropic regularizer
\begin{equation}
\label{eq:entro-reg}
    F(x) = \frac12 \sum_{i=1}^d \bigl( (1+x_i) \log (1+x_i) + (1-x_i) \log (1-x_i) \bigr)
\end{equation}
together with the following perturbation of a point $a_t$ in the interior of $\cA$:
\begin{quote}
With probability $\gamma$, play $\tilde{a}_t$ uniformly at random from the canonical basis (with random sign).
With probability $1-\gamma$, play $\tilde{a}_t = \xi_t$ where $\xi_t(i)$ is drawn from a Rademacher with parameter $\frac{1+a_t(i)}{2}$.
\end{quote}
It is easy to check that this perturbation is almost unbiased, indeed one has:
$$\E\,\tilde{a}_t(i) = (1-\gamma) \left(\frac{1+a_t(i)}{2} - \frac{1 - a_t(i)}{2} \right) =  (1-\gamma) a_t(i), $$
and thus:
\begin{equation} \label{eq:hypercubeunbiased1}
\big\|\E[\tilde{a}_t \mid a_t] - a_t\big\|_{\infty} \leq \gamma.
\end{equation}
We can now prove the following result.
\begin{theorem} \label{th:hypercube}
Consider the online linear optimization problem with bandit feedback on $\cA = \{x \in \R^d: \|x\|_{\infty} \leq 1 \}$, and with $\cZ = \{x \in \R^d: \|x\|_{1} \leq 1 \}$. Then {\sc osmd} on $\cA$ with regularizer~(\ref{eq:entro-reg}) satisfies, for any $\eta$ and $\gamma \in (0,1)$ such that $\frac{\eta d}{\gamma} \leq \frac12$,
\begin{equation} \label{eq:hypercuberegret1}
R_n \leq \gamma n + \frac{d \log 2}{\eta} + \eta \sum_{t=1}^n \sum_{i=1}^d \E \Bigl[ \bigl(1 - a_t(i)^2\bigr) \tilde{z}_t(i)^2 \Bigr] .
\end{equation}
In particular, with $\gamma = 2 d \sqrt{\frac{\log 2}{3 n}}$ and $\eta= \sqrt{\frac{\log 2}{3 n}}$,
\begin{equation} \label{eq:hypercuberegret2}
R_n \leq 2 d \sqrt{3 n \log 2} .
\end{equation}
\end{theorem}
Remark that the regularizer~(\ref{eq:entro-reg}) used here is in the class of Legendre functions with exchangeable Hessian. More precisely, following \cite{ABL11}, (\ref{eq:entro-reg}) can be written (up to a numerical constant) as
\[
    F(x) = \sum_{i=1}^d \int_{-1}^{x_i} \tanh^{-1}(s) ds~.
\]
This type of regularizer was first studied (implicitely) by \cite{AB09} and \cite{AB10}.
\begin{proof}
Since $F$ is Legendre on $\cA$, $F^*$ is differentiable on $\R^d$ and the gradient mapping of $F^*$ is the inverse of the gradient mapping of $F$. Therefore, $(\nabla F^*)_i = \tanh$ because $(\nabla F^*)_i = \tanh^{-1}$. Then, thanks to~\eqref{eq:hypercubeunbiased1} and Theorem \ref{th:osmd}, the regret can be bounded as:
$$\gamma n + \frac{\sup_{a \in \cA} F(a) - F(a_1)}{\eta} + \frac{1-\gamma}{\eta} \sum_{t=1}^n \E\,D_{F^*}\bigl( - \eta \tilde{z}_1^t,  - \eta \tilde{z}_1^{t-1}\bigr) .$$
For the first term it is easy to see that $F(a) - F(a_1) \leq d \log 2$. For the term involving the Bregman divergence, using elementary computations one obtains
$$D_{F^*}(u,v) = \sum_{i=1}^d \left( \log \frac{\cosh(u_i)}{\cosh(v_i)} - \tanh(v_i) (u_i - v_i) \right) .$$
To prove \eqref{eq:hypercuberegret1} we need to show that $D_{F^*}(u,v) \leq \sum_{i=1}^d \bigl(1 - \tanh^2(v_i)\bigr) (u_i-v_i)^2$. In fact, we prove that this inequality is true as soon as $\|u-v\|_{\infty} \leq \frac12$. The fact that the property is satisfied for the pair $(u,v) = \bigl( - \eta \tilde{z}_1^t,  - \eta \tilde{z}_1^{t-1}\bigr)$ under consideration is established at the very end of the proof.

Using a basic hyperbolic identity, and the elementary inequalities $\exp(x) \leq 1+x+x^2, \forall x: |x| \leq 1$ and $\log(1+x) \leq x$, one obtains
{\allowdisplaybreaks
\begin{align*}
& \log \left( \frac{\cosh(u_i)}{\cosh(v_i)} \right) - \tanh(v_i) (u_i - v_i)  \\
& =  \log \left( \frac{\cosh(v_i) \cosh(u_i - v_i) + \sinh(v_i) \sinh(u_i - v_i)}{\cosh(v_i)} \right) - \tanh(v_i) (u_i - v_i)\\
& =  \log \bigg( \cosh(u_i - v_i) + \tanh(v_i) \sinh(u_i - v_i) \bigg) - \tanh(v_i) (u_i - v_i) \\
& =  \log \bigg( \frac{1+\tanh(v_i)}{2} \exp(u_i - v_i) + \frac{1- \tanh(v_i)}{2} \exp(- (u_i - v_i)) \bigg)
\\ &\quad - \log \exp \bigg( \tanh(v_i) (u_i - v_i)\bigg) \\
& =  \log \bigg(\frac{1+\tanh(v_i)}{2} \exp\big( (1-\tanh(v_i)) (u_i-v_i) \big)
\\ &\quad + \frac{1-\tanh(v_i)}{2} \exp\big(- (1+\tanh(v_i)) (u_i-v_i) \big)  \bigg) \\
& \leq  \log \big( 1 + (1 - \tanh^2(v_i)) (u_i-v_i)^2 \big)
\leq   (1 - \tanh^2(v_i)) (u_i-v_i)^2
\end{align*}
}
which concludes the proof of \eqref{eq:hypercuberegret1}. Now for the proof of \eqref{eq:hypercuberegret2} we first compute the matrix $P_t$:
\begin{eqnarray*}
\E\,\tilde{a}_t \tilde{a}_t^{\top} & = & \frac{\gamma}{d} I_d + (1-\gamma) \sum_{i,j =1}^d \E\,\xi_t(i) \xi_t(j)\, e_i e_j^{\top} \\
& = &  \frac{\gamma}{d} I_d + (1-\gamma) I_d + (1-\gamma) \sum_{i \neq j} \E\,\xi_t(i) \xi_t(j)\,e_i e_j^{\top} \\
& = &  \frac{\gamma}{d} I_d + (1-\gamma) I_d + (1- \gamma) \sum_{i \neq j} a_t(i) a_t(j)\,e_i e_j^{\top} \\
& = &  \frac{\gamma}{d} I_d + (1-\gamma) a_t a_t^{\top} + (1-\gamma) \sum_{i=1}^d \bigl(1 - a_t(i)^2 \bigr) e_i e_i^{\top}. \\
\end{eqnarray*}
To obtain \eqref{eq:hypercuberegret2} first note that
$(1- \gamma) \sum_{i=1}^d \E \bigl[ (1 - a_t(i)^2) \tilde{z}_t(i)^2 \bigr] \leq \E \ \tilde{z}_t^{\top} P_t \tilde{z}_t.$
Now we use a spectral decomposition of $P_t$ in an orthonormal basis and write:
$P_t = \sum_{i=1}^d \lambda_i v_i v_i^{\top} .$
In particular we have $P_t^{-1} = \sum_{i=1}^d \frac{1}{\lambda_i} v_i v_i^{\top}$ and thus:
$$\E\,\tilde{a}_t^{\top} P_t^{-1} \tilde{a}_t  =  \sum_{i=1}^d \frac{1}{\lambda_i} \E\,\tilde{a}_t^{\top} v_i v_i^{\top} \tilde{a}_t
=  \sum_{i=1}^d \frac{1}{\lambda_i} v_i^{\top} P_t v_i
=  \sum_{i=1}^d \frac{1}{\lambda_i} \lambda_i v_i^{\top} v_i
=  d.$$
To conclude the proof it remains now to show that 
$\eta ||\tilde{z}_t||_{\infty} \leq \frac12.$
First note that the smallest eigenvalue of $P_t$ is larger than $\gamma/d$, and thus:
$$\eta |\tilde{z}_t(i)| = \eta \big| e_i^{\top} P_t^{-1} \tilde{a}_t \tilde{a}_t^{\top} z_t\big| 
 \leq  \eta \big| e_i^{\top} P_t^{-1} \tilde{a}_t \big|
 \leq  \frac{\eta d}{\gamma} 
 \leq  \frac12$$
where the penultimate inequality follows from $|e_i^{\top} \tilde{a}_t| \leq 1$ and the last inequality follows from the assumption on $\eta$ and $\gamma$.
\end{proof}

\section{Improved regret for the Euclidean ball} \label{sec:ball}
In this section we restrict our attention to the action set $\cA = \{x \in \R^d: \|x\| \leq 1 \}$, where $\|\cdot\|$ denotes the Euclidean norm. Using {\sc exp2} with John's exploration on a discretization of the Euclidean ball one obtains a regret bound of order $d \sqrt{n \log n}$ for this problem. A similar regret bound can be obtained with a computationally efficient algorithm, using the technique developed by \cite{AHR08}. Here we show that in fact one can attain efficiently a regret of order $\sqrt{d n \log n}$ using {\sc osmd} with the approriate regularizer $F$ and random perturbation $\tilde{a}_t$. More precisely here we use $F(x) = - \log (1 - \|x\|) - \|x\|$ (the motivation for this particular regularizer comes from the proof, see below). Moreover we perform the following perturbation of a point $a_t$ in the interior of $\cA$:
\begin{quote}
Let $\xi_t$ be a Bernoulli of parameter $\|a_t\|$, let $I_t$ be drawn uniformly at random in $\{1, \hdots, d\}$, and let $\epsilon_t$ be Rademacher with parameter $\frac12$. If $\xi_t = 1$, then play $\tilde{a}_t = a_t / \|a_t\|$, else play $\tilde{a}_t = \epsilon_t e_{I_t}$.
\end{quote}
It is easy to check that this perturbation is unbiased, in the sense that $\E\bigl[\tilde{a}_t \mid a_t\bigr] = a_t$.
Here we modify the estimate of step (b) in Figure \ref{fig:OSMD}, and instead we use:
\begin{equation} \label{eq:ballestimate}
\tilde{z}_t = (1 - \xi_t) \frac{d}{1- \|a_t\|} (z_t^{\top} \tilde{a}_t) \tilde{a}_t.
\end{equation}
It is easy to check that this estimator satisfies the same key unbiasedness property than the one in step (b) in Figure \ref{fig:OSMD}, that is $\E\bigl[\tilde{z}_t \mid a_t\bigr] = z_t$.
%
\newline

Note that the problem studied in this section was also specifically considered in \cite{AR09}, with an emphasis on high probability bounds. In this paper the authors used the self-concordant barrier $F(x) = - \log (1 - \|x\|^2)$ with a similar perturbation scheme to the one proposed above. They obtain suboptimal rates, but a more careful analysis (precisely slightly modifying Section V.B., step (E)) can actually yield the same rate than the one we obtain. The strength of our approach is that it is in a sense more elementary (e.g., we do not require any results from the Interior Point Methods literature), but on the other hand the result of \cite{AR09} holds with high probability (though it is not clear if it possible to get the rate $\sqrt{d n \log n}$ with high probability).
\begin{theorem} \label{th:ball}
Consider the online linear optimization problem with bandit feedback on $\cA = \{x \in \R^d: \|x\| \leq 1 \}$, and with $\cZ = \{x \in \R^d: \|x\| \leq 1 \}$. Then {\sc osmd} on $\cA' = \{x \in \R^d: \|x\| \leq 1-\gamma \}$ with the estimate \eqref{eq:ballestimate}, and $F(x) = - \log (1 - \|x\|) - \|x\|$ satisfies, for any $\eta$ such that $\eta d \leq \frac12$,
\begin{equation} \label{eq:ballregret1}
R_n \leq \gamma n + \frac{\log \gamma^{-1}}{\eta} + \eta \sum_{t=1}^n \E\Bigl[\bigl(1 - \|a_t\|\bigr) \|\tilde{z}_t\|^2\Bigr] .
\end{equation}
In particular, with $\gamma = \frac{1}{\sqrt{n}}$ and $\eta= \sqrt{\frac{\log n}{2 n d}}$,
\begin{equation} \label{eq:ballregret2}
R_n \leq  3 \sqrt{d n \log n} .
\end{equation}
\end{theorem}
\begin{proof}
First, it is clear that by playing on $\cA'$ instead of $\cA$, one incurs an extra $\gamma n$ regret. Second, note that $F$ is stricly convex (it is the composition of a convex and nondecreasing function with the euclidean norm), differentiable, and
\begin{equation} \label{eq:ballgradient}
\nabla F (x) = \frac{x}{1 - \|x\|}~.
\end{equation}
In particular $F$ is Legendre on $\cA = \{x \in \R^d: \|x\| \leq 1 \}$, and thus $F^*$ is differentiable on $\R^d$. Now the regret with respect to $\cA'$ can be bounded as follows, thanks to Theorem \ref{th:osmd}, 
$$\frac{\sup_{a \in \cA'} F(a) - F(a_1)}{\eta} + \frac{1}{\eta} \sum_{t=1}^n \E\,D_{F^*}\bigg(\nabla F(a_t) - \eta \tilde{z}_t, \nabla F(a_t)\bigg) .$$
The first term is clearly bounded by $\tfrac{1}{\eta}\log\tfrac{1}{\gamma}$ (we use the fact that $a_1 = 0$). For the second term we need to do a few computations (the first one follows from \eqref{eq:ballgradient} and the fact that $F$ is Legendre):
\begin{eqnarray*}
\nabla F^* (u) & = & \frac{u}{1 + \|u\|},\\
F^*(u) & = & - \log (1 + \|u\|) + \|u\|,\\
D_{F^*}(u, v) & = & \frac{1}{1+\|v\|} \left(\|u\| - \|v\| + \|u\| \cdot \|v\| - v^{\top} u - (1+\|v\|) \log \left( 1 + \frac{\|u\| - \|v\|}{1 + \|v\|}\right) \right).
\end{eqnarray*}
Let $\Theta(u,v)$ such that $D_{F^*}(u, v) =\frac{1}{1+\|v\|}  \Theta(u,v)$. First note that 
\begin{equation} \label{eq:superball}
\frac{1}{1+\|\nabla F (a_t)\|} = 1 - \| a_t \|~.
\end{equation}
Thus, in order to prove \eqref{eq:ballregret1} it remains to show that $\Theta(u,v) \leq \|u - v\|^2$, for $(u,v) = \bigl( - \eta \tilde{z}_1^t,  - \eta \tilde{z}_1^{t-1}\bigr)$. In fact we shall prove that this inequality holds true as soon as $\frac{\|u\| - \|v\|}{1 + \|v\|} \geq - \frac{1}{2} .$ This is the case for the pair $(u,v)$ under consideration, since by the triangle inequality, equations \eqref{eq:ballestimate} and \eqref{eq:superball}, and the assumption on $\eta$:
$$\frac{\|u\| - \|v\|}{1 + \|v\|}  \geq - \frac{\eta \|\tilde{z}_t\|}{1 + \|v\|} \geq - \eta d \geq - \frac{1}{2}.$$
Now using that $\log(1+x) \geq x -x^2$, $\forall x \geq - \frac12$, we obtain that for $u,v$ such that $\frac{\|u\| - \|v\|}{1 + \|v\|} \geq - \frac{1}{2}$,
\begin{eqnarray*}
\Theta(u,v) & \leq & \frac{(\|u\| - \|v\|)^2}{1 + \|v\|} + \|u\| \cdot \|v\| - v^{\top} u \\ 
& \leq & (\|u\| - \|v\|)^2 + \|u\| \cdot \|v\| - v^{\top} u \\
& = & \|u\|^2 + \|v\|^2 - \|u\| \cdot \|v\| - v^{\top} u \\
& = & \|u - v\|^2 + 2 v^{\top} u - \|u\| \cdot \|v\| - v^{\top} u \\
& \leq & \|u - v\|^2
\end{eqnarray*}
which concludes the proof of \eqref{eq:ballregret1}. Now for the proof of \eqref{eq:ballregret2} it suffices to note that:
$$\E\Bigl[\bigl(1 - \|a_t\|\bigr) \|\tilde{z}_t\|^2\Bigr] = (1- \|a_t\|) \sum_{i=1}^d  \frac{1 - \|a_t\|}{d} \frac{d^2}{(1 - \|a_t\|)^2} (z_t^{\top} e_i)^2 = d \|z_t\|^2 \leq d$$
along with straightforward computations.
\end{proof}

\subsection*{Acknowledgements} 
The first author would like to thank Csaba Szepesv{\'a}ri for bringing to his attention the problem of optimal regret on the Euclidean ball, as well as Alexander Rakhlin for illuminating discussions regarding sampling schemes. He also thank Ramon Van Handel, Vianney Perchet and Philippe Rigollet for stimulating discussions on this topic.

\bibliographystyle{plainnat}
\bibliography{newbib}
\end{document}